\documentclass[journal]{IEEEtran}

\IEEEoverridecommandlockouts                              % This command is only
\usepackage{cite}
\usepackage{enumerate}
\usepackage[english]{babel}
\usepackage[utf8]{inputenc}
\usepackage{algorithm}
\usepackage[noend]{algpseudocode}
\usepackage{amsmath}
\usepackage{amssymb}
\usepackage{amsthm}
\usepackage{graphicx}
\usepackage{hyperref}
\usepackage{graphics}
\newtheorem{definition}{Definition}
\newtheorem{lemma}{Lemma}
\newtheorem{theorem}{Theorem}

\newtheorem{remark}{Remark}

\newcommand{\R}{\mathbb{R}}
\newcommand{\E}{\mathbb{E}}

\newcommand{\g}{\gamma}

\usepackage{caption}
\usepackage{subcaption}

\title{
A Convergent Off-Policy Temporal Difference Algorithm 
}

\author{Raghuram Bharadwaj Diddigi$^{*}$ Chandramouli Kamanchi$^{*}$ Shalabh Bhatnagar% <-this % stops a space
\thanks{$^{*}$ Equal Contribution.}% <-this % stops a space
\thanks{This work was supported by the Robert Bosch Centre for Cyber-Physical Systems, Indian Institute of Science and a
grant from the Department of Science and Technology, Government of India.}
\thanks{R. B. Diddigi and C. Kamanchi are with the Department of Computer
Science and Automation, Indian Institute of Science, Bengaluru 560012,
India (e-mail: raghub@iisc.ac.in; chandramouli@iisc.ac.in.).}%
\thanks{S. Bhatnagar is with the Department of Computer Science and
Automation and the Department of Robert Bosch Centre for Cyber-Physical
Systems, Indian Institute of Science, Bengaluru 560012, India (e-mail:
shalabh@iisc.ac.in).}%
}

\begin{document}

\maketitle
\thispagestyle{empty}
\pagestyle{empty}
%%%%%%%%%%%%%%%%%%%%%%%%%%%%%%%%%%%%%%%%%%%%%%%%%%%%%%%%%%%%%%%%%%%%%%%%%%%%%%%%
\begin{abstract}
Learning the value function of a given policy (target policy) from the data samples obtained from a different policy (behavior policy) is an important problem in Reinforcement Learning (RL). This problem is studied under the setting of off-policy prediction. Temporal Difference (TD) learning algorithms are a popular class of algorithms for solving the prediction problem. TD algorithms with linear function approximation are shown to be convergent when the samples are generated from the target policy (known as on-policy prediction). However, it has been well established in the literature that off-policy TD algorithms under linear function approximation diverge. In this work, we propose a convergent on-line off-policy TD algorithm under linear function approximation. The main idea is to penalize the updates of the algorithm in a way as to ensure convergence of the iterates. We provide a convergence analysis of our algorithm. Through numerical evaluations, we further demonstrate the effectiveness of our algorithm.    
\end{abstract}
% \begin{IEEEkeywords}
% Minimax Q-learning, , stochastic approximation, zero-sum games.
% \end{IEEEkeywords}

\section{Introduction}
The two important problems in Reinforcement Learning (RL) are Prediction and Control \cite{bertsekas1996neuro}. The prediction problem deals with computing the value function of a given policy. In a discounted reward setting, value function refers to the total expected discounted reward obtained by following the given policy. The control problem refers to computing the optimal policy, i.e., the policy that maximizes the total expected discount reward. When the model information (probability transition matrix and single-stage reward function) is fully known, techniques like value iteration and policy iteration are used to solve the control problem. Policy iteration is a two-step iterative algorithm where the task of prediction is performed in the first step for a given policy followed by the policy improvement task in the second step. However, in most of the practical scenarios, the model information is not known and instead (state, action, reward and next-state) samples are only available. Under such a model-free setting, popular RL algorithms for prediction are Temporal Difference (TD) and for control are Q-Learning and Actor-Critic algorithms \cite{sutton2018reinforcement}. Actor-Critic algorithms can be seen as model-free analogs of the policy iteration algorithm and involve a model-free prediction step. Therefore, it is clear that model-free prediction is an important problem for which optimal and convergent solutions are desired.

TD algorithms under the tabular approach (where there is no approximation of the value function) are a very popular class of algorithms for computing the exact value function of a given policy (henceforth referred to as target policy) from samples. In many of the real-life problems though, we encounter situations where the number of states is large or even infinite. In such cases, it is not possible to use tabular approaches and one has to resort to approximation based methods. TD algorithms are shown to be stable and convergent under linear function approximation, albeit under the setting of on-policy. On-policy refers to the setting where state and action samples are obtained using the target policy itself.
%That is, states are sampled from the stationary distribution of the target policy, actions are picked from target policy leading to a reward and next state. 
As we approach practical scenarios, it can be noted that such samples are not always available to the practitioner. For example in games, say a practitioner would like to evaluate a (target) strategy. However, the data available to her might be from a player following a different strategy. The question that arises in this scenario is whether she can make use of this data and still evaluate the target strategy. These problems are studied under the setting of off-policy prediction where the goal is to evaluate the value function of the target policy from the data generated from a different policy (commonly referred to as behavior policy). The recent empirical success of the Deep Q-Learning algorithm (model-free control algorithm) motivates us to understand its convergence behavior, which is a very difficult problem. It has been noted in (Section 11.3 of \cite{sutton2018reinforcement}) that convergence and stability issues arise when we combine three components - function approximation, bootstrapping (TD algorithms) and off-policy learning, what they refer to as the ``deadly triad". 

In our work, we propose an online off-policy stable TD algorithm for a prediction problem under linear function approximation. The idea is to penalize the parameters of the TD update to mitigate the divergence problem. We note here that the recent work \cite{ghiassian2018online} provides a comprehensive and excellent survey of algorithms for off-policy prediction problems and performs a comparative study. However, for the sake of completeness, we now discuss some of the important and relevant works on the off-policy prediction problem. In \cite{bradtke1996linear}, Least-Squares TD algorithms (LSTD) with linear function approximation have been proposed that are shown to be convergent under both on-policy and off-policy settings. However, the per-step complexity of LSTD algorithms is quadratic in the number of parameters. In \cite{precup2001off}, off-policy TD algorithms are proposed that make use of an importance sampling idea to convert the expected value of total discounted reward under behavior policy to expected value under target policy. However, the variance of such algorithms is very high and in some cases tends to be infinite. In \cite{sutton2008convergent}, the Gradient TD (GTD) algorithm has been proposed that is stable under off-policy learning and linear approximation and has linear (in the number of parameters) complexity. Since then, there have been a lot of improvements on the GTD algorithm under various settings like prediction, control, and non-linear function approximation \cite{sutton2009fast,bhatnagar2009convergent,maei2010gq,maei2010toward}. The idea of adding the penalty in the form of a regularization term has been considered in \cite{liu2012regularized} where Regularized off-policy TD (RO-TD) algorithm has been proposed based on GTD algorithms and convex-concave saddle point formulations. Emphatic TD algorithms (ETD) \cite{sutton2016emphatic,yu2015convergence,hallak2016generalized,ghiassian2017first} are another popular class of off-policy TD algorithms that achieve stability by emphasizing or de-emphasizing updates of the algorithm. These updates also have linear-time complexity. Moreover, these algorithms learn only one set of parameters, unlike GTD algorithms which are two-time scale stochastic approximation algorithms that learn two sets of parameters. Recently in \cite{hallak2017consistent,gelada2019off}, a co-variance off-policy TD (COP-TD) algorithm has been proposed that includes a co-variance shift term in the TD update. This shift term is also learned along with the parameter of the algorithm.  

Our algorithm, like the Emphatic TD algorithm, trains only one set of parameters and like ETD and GTD algorithms, has per-update complexity that is linear in the number of parameters. The contributions of our paper are as follows:
\begin{itemize}
    % \item We propose a novel 
    % regularized projective bellman error metric that enables us to propose convergent off-policy TD learning algorithms.
    \item We derive an online off-policy TD learning algorithm with linear function approximation. Our algorithm has linear per-iteration computational complexity in the number of parameters.
    \item We prove the convergence of our algorithm utilizing the techniques of \cite{tsit_roy,sutton2016emphatic}.
    \item We show the empirical performance of our algorithm on standard benchmark RL environments.
\end{itemize}
The rest of the paper is organized as follows. In Section II, we introduce the background and preliminaries. We propose our algorithm in Section III. Sections IV and V describe the analysis of our algorithm. Section VI presents the results of our numerical experiments. Finally, Section VII presents concluding remarks and future research directions.

\section{Background and Preliminaries}
We consider a Markov Decision Process (MDP) of form $(S,U,p,r,\gamma)$ where $S$ denotes the state space. $U$ is the set of actions, $p$ is a probability transition matrix where $p(s'|s,a)$ denotes the probability of system transition to state $s'$ when action $a$ is chosen in state $s$. $r$ is the single-stage reward where $r(s,a)$ denotes the reward obtained by taking action $a$ in state $s$. Finally, $\gamma$ denotes the discount factor. Let $\pi: S \xrightarrow{} \Delta(U)$ be the target policy where $\Delta(U)$ denotes the set of probability distributions over actions. The objective of the MDP prediction problem is to estimate the value function ($V^{\pi}$) of the target policy $\pi$, where the value function of a state $s \in S$ denoted by $V^{\pi}(s)$ is given by: 
\begin{align}
    V^\pi(s) = \E \Big{[} \sum_{i=0}^{\infty} \gamma^i r(s_i, a_i) \Big{|} s_0 = s, \pi \Big{]},
\end{align}
where the state-action trajectory $(s_0,a_0,s_1,\ldots)$ is obtained following the policy $\pi$ and $\E[.]$ denotes the expectation. 

As the number of states of the MDP can be very large, we resort to approximation techniques to compute the value function. In our work, we consider the linear function approximation architecture where 
\begin{align} \label{app-v}
    \widehat{V}(s) = \theta^{T}\phi(s),
\end{align}
where $\widehat{V}(s)$ denotes the approximate value function associated with state $s$ (that we desire to be very close to the exact value function), $\phi(s)$ is a $d\times1$ feature vector associated with state $s$ and $\theta$ is a $d\times1$ weight vector. Note that the exact value function $V^{\pi}$ may not be representable by \eqref{app-v}. Therefore, the objective is to estimate the weight vector $\theta$ so that the approximate value function denoted by \eqref{app-v} is as close as possible to the exact value function. 

The on-policy TD(0) \cite{sutton2018reinforcement} is a popular on-line algorithm for computing the weight vector $\theta$. The update equation is given by:
\begin{align}\label{optd}
    \theta_{n+1} = \theta_{n} + \alpha_n(r_{n}+ \gamma \theta_n^{T}\phi(s_{n+1})  - \theta_n^{T}\phi(s_n))\phi(s_n),
\end{align}
where $(s_n,r_n,s_{n+1})$ is the state, reward and next state samples obtained at time $n$, $\alpha_n, ~ n \geq 0$ is the step-size sequence and $\theta_0$ denotes the initial parameter vector.  

The stability of the on-policy TD(0) algorithm is well established in the literature \cite{sutton2016emphatic}. We outline the proof of the convergence of this algorithm. Following the notation of \cite{sutton2016emphatic}, please note that the update rule \eqref{optd} can be re-written as:
\begin{align}\label{ropue}
     \theta_{n+1} = \theta_{n} + \alpha_n (b_n - A_n\theta_n),
\end{align}
where $A_n = \phi(s_n)(\phi(s_n) - \gamma \phi(s_{n+1}))^{T}$ and $b_n = r_{n+1}\phi(s_n)$.

It is shown in \cite{sutton2018reinforcement} that the algorithm with update rule \eqref{ropue} is stable if the matrix $A$ given by:
\begin{align}\label{keyM}
    A = \lim_{n \xrightarrow{} \infty} A_n = \Phi^{T}D_{\pi}(I - \gamma P_{\pi})\Phi
\end{align}
is positive definite. In \eqref{keyM}, $\Phi$ is a $|S| \times d$ matrix with the feature vector $\phi(s)$ in row $s$. $D_\pi$ is the $|S| \times |S|$ diagonal matrix with the diagonal being the stationary distribution of state $i$ obtained under policy $\pi$. Finally, $P_\pi$ is a $|S| \times |S|$ matrix with $[P_\pi]_{ij} = \sum_{a} \pi(i,a)p(j|i,a)$. For the on-policy TD(0) algorithm, $A$ is shown to be positive definite \cite{sutton2016emphatic} proving the stability of the algorithm. 

In the off-policy prediction problem, the data samples are obtained from a behavior policy $\mu$ instead of the target policy $\pi$. In this case, the off-policy TD(0) update \cite{sutton2016emphatic} is given by:
\begin{align}
    \theta_{n+1} = \theta_{n} + \alpha_n \rho_n \left(r_{n}+ \gamma \theta_n^{T}\phi(s_{n+1})  - \theta_n^{T}\phi(s_n)\right)\phi(s_n),
\end{align}
where $r_n$ is the reward obtained by taking action $a_n$ in state $s_n$ and $\rho_n$ is the importance sampling ratio given by $\frac{\pi(s_n,a_n)}{\mu(s_n,a_n)}$. The corresponding matrix $A$ for this algorithm is given by:
\begin{align}\label{keyoff}
     A = \Phi^{T}D_{\mu}(I - \gamma P_{\pi})\Phi,
\end{align}
where $D_{\mu}$ is a diagonal matrix with diagonal of $D_\mu$ being the stationary distribution obtained under policy $\mu$. 

The matrix $A$ defined in \eqref{keyoff} need not be positive definite \cite{sutton2016emphatic}. Therefore stability and convergence of the off-policy TD(0) are not guaranteed.

The off-policy TD(0) algorithm, if it converges, may perform comparably to some of the off-policy convergent algorithms in the literature. For example, in Figure 5 of \cite{ghiassian2018online}, it has been shown that the performance of off-policy TD(0) is comparable to the GTD(0) algorithm. However, as the algorithm is not stable, off-policy TD(0) can diverge. In this paper, we propose a simple and stable off-policy TD algorithm. In the next section, we propose our algorithm and in Section IV, we provide its convergence analysis.
\section{The Proposed Algorithm} \label{PA}
% \begin{algorithm}[H] 
\begin{algorithm}[t] 
\caption{Perturbed Off-Policy Prediction Algorithm}\label{alg:OPA}
\hspace*{\algorithmicindent} \textbf{Input:}\\
\hspace*{\algorithmicindent} $\mu,\pi$: behaviour and target policies respectively \\
\hspace*{\algorithmicindent} $(s_n,a_n,r_n)^{\infty}_{n=0}$: data from behaviour policy \\
\hspace*{\algorithmicindent} $\theta_0$: initial parameter vector \\
\hspace*{\algorithmicindent} $\g$: discount factor\\
\hspace*{\algorithmicindent} $\phi(s)$: feature vector of state $s$\\
\hspace*{\algorithmicindent} $\eta$: non-negative regularization parameter\\
\hspace*{\algorithmicindent} $\{\alpha_n\}$: step-size sequence\\
\hspace*{\algorithmicindent} Iter: total number of iterations\\
\hspace*{\algorithmicindent} \textbf{Output:}  $\theta_{\text{Iter}}$
\begin{algorithmic}[1]
\Procedure{Off-Policy Prediction:}{}
\While{$n < $ Iter}
\State $\rho_{n}=\frac{\pi(s_n,a_n)}{\mu(s_n,a_n)}$
% \State $\delta_n=r_n+\g\phi(s_{n+1})^T\theta_n-\phi(s_{n})^T\theta_n-\eta\phi(s_{n})^T\theta_n$
\State $\delta_n=r_n+\g\phi(s_{n+1})^T\theta_n-(1+\eta)\phi(s_{n})^T\theta_n$
\State $\theta_{n+1}=\theta_n+\alpha_n\rho_n\delta_n \phi(s_n) $
\EndWhile
\State \textbf{return} $\theta_{\text{Iter}}$ 
\EndProcedure
\end{algorithmic}
\label{algo:PA}
\end{algorithm}

The input to our algorithm is the target policy, whose value function we want to estimate and the behavior policy, using which the samples are generated. Also, provided as an input to our algorithm is the regularization parameter ($\eta$). The algorithm works as follows. At each time step $n$, we obtain a sample $(s_n,a_n,r_n,s_{n+1})$ using which importance sampling is computed as shown in the Step 3. We then compute our modified temporal difference term as show in Step 4. Finally, the parameters of the algorithm are updated as shown in Step 5. 

\begin{remark}
It is clear from the Algorithm \ref{algo:PA} that, the per-step complexity is $\mathcal{O}(d)$, where $d$ is the number of parameters.
\end{remark}

% The algorithm given behaviour and target policies proceeds to calculate importance sampling weight $\rho_n$ and updates parameter $\theta_n$ by estimating temporal differences $\delta_n$ as shown in steps 4 and 5 of the algorithm in an online manner. 
% Our algorithm unlike GTD family of algorithms and similar to Emphatic TD algorithm has linear per iteration complexity in the number of parameters.
\begin{remark}
The choice of $\eta$ is critical in our algorithm. Larger values of $\eta$ ensure convergence (see Lemma \ref{eta_bound}) and smaller values of $\eta$ ensure more accurate solution (see Lemma \ref{T_approx})  
\end{remark}
In the next section, we provide the convergence analysis of our proposed algorithm. 
\section{Convergence Analysis}
The update rule of Algorithm \ref{algo:PA} can be rewritten as follows.
\begin{align*}
    \theta_{n+1}&=\theta_n+\alpha_n\rho_n\delta_n \phi(s_n)\\
    &=\theta_n+\alpha_n(b_n-A_n\theta_n), 
\end{align*}
where $A_n$ and $b_n$ are given by
\begin{align}
    A_{n}&=-\rho_n\big{(}\gamma\phi(s_n)\phi(s_{n+1})^T-(1+\eta)\phi(s_n)\phi(s_n)^T\big{)}, \label{A_n} \\
    b_{n}&=\rho_nr_n\phi(s_n). \label{b_n}
\end{align}
We state and invoke Theorem 2 (also see Th. 17, p. 239 of  \cite{benveniste2012adaptive}) of \cite{tsit_roy} to show the convergence of our algorithm.
\begin{theorem}
Consider an iterative algorithm of the form
\begin{align*}
\theta_{n+1}=\theta_{n}+a_{n}\left(b(X_n)-A(X_n)\theta_{n}\right)
\end{align*}
where
\begin{enumerate}
    \item the step-size sequence satisfies $\sum a_n=\infty$, $\sum a^2_n < \infty.$
    \vspace*{0.1 cm}
    \item $X_{n}$ is a Markov process with a unique invariant distribution.
    \vspace*{0.1 cm}
    \item $A=\E_0[A(X_n)]$ and $b=\E_0[b(X_n)].$ Here $\E_0$ is the expectation with respect to the stationary distribution of the Markov chain given by the behaviour policy $\mu.$
    \vspace*{0.1 cm}
    \item The matrix $A$ is positive definite.
    \vspace*{0.1 cm}
    \item There exist positive constants $C,q$ and a positive real valued function $h$ from the states of the Markov chain $\{X_{n}\}$ such that $\sum^{\infty}_{n=0}\|\E[A(X_{n})|X_{0}=X]-A\|\leq C\left(1+h^q(X)\right)$ and 
    $\sum^{\infty}_{n=0}\|\E[b(X_{n})|X_{0}=X]-b\|\leq C\left(1+h^q(X)\right).$
    \vspace*{0.1 cm}
    \item For any $q>1$ there exists a constant $\kappa_{q}$ such that for all $X$ and $n$, 
    $\E[h^q(X_{n})|X_{0}=X]\leq \kappa_{q}\left(1+h^q(X)\right).$
\end{enumerate}
Under these assumptions, i.e. 1-6 above, it is known that $\theta_n$ converges to the solution of $b-A\theta=0.$
\end{theorem}
To begin we define the process $X_n$ as follows. Let $X_n=(s_n,a_n,s_{n+1}).$ Observe that $X_n$ is a Markov chain. In particular, $s_{n+1}$ is a deterministic function of $X_n$ and the distribution of $a_{n+1}$ and $s_{n+2}$ depends only on $s_{n+1}$.
Also note that, in our algorithm, $A(X_n)=A_n$ and $b(X_n)=b_n$ given by equations \eqref{A_n} and \eqref{b_n} respectively with $X_n=(s_n,a_n,s_{n+1})$. 

Assumptions 1 and 2 are fairly general. The assumptions 5 and 6 can be shown to hold with $h$ as a constant function for finite state-action MDPs and the arguments are similar to those in theorem 1 of \cite{tsit_roy}. Therefore, the most important assumption to verify is that the matrix $A$ is positive definite. In this section, we prove that $A$ is positive definite, thereby proving the convergence of our proposed algorithm. 

We begin by proving some important lemmas that are used in our main theorem. 
% Here we show that $A$ is positive definite for a suitable choice of $\eta$ and the rest of the analysis follows from the arguments of \cite{tsit_roy}
\begin{lemma}
Let $\Phi$ be the $|S| \times d$ matrix where the $i^{th}$ row of $\Phi$ is given by $\phi(i)$, the feature vector of state $i$ and $r_{\pi}$ be the $|S| \times 1$ vector where the $i^{th}$ component is given by $r_{\pi}(i)=\sum_{a \in A}r(i,a)\pi(i,a)$.
Let $\E_0$ be the expectation with respect to the stationary distribution of the Markov chain realized by $\mu$. Then
$A=\E_0[A_{n}]$ and $b=\E_0[b_{n}]$ are given by
\begin{align*}
A &=\Phi^TD_{\mu}\left((1+\eta)I-\gamma P_{\pi}\right)\Phi,\\
b &= \Phi^TD_{\mu}r_{\pi},
\end{align*}
where $D_{\mu}$ is a diagonal matrix with the $i^{th}$ diagonal element being $d_{\mu}(i)$.
\end{lemma}
\begin{proof}
\begin{align*}
\E_{0}[A_n]&=-\E_{0}\left[\rho_n\big{(}\gamma\phi(s_n)\phi(s_{n+1})^T-(1+\eta)\phi(s_n)\phi(s_n)^T\big{)}\right]\\
&=-\sum_{i,j \in S, a \in U}\mu(i,a)\bigg{[} \frac{\pi(i,a)}{\mu(i,a)} \big{(}\g\phi(i)\phi(j)^T d_{\mu}(i) p(j|i,a)\\
& \hspace{3cm}-(1+\eta)d_{\mu}(i)\phi(i)\phi(i)^T\big{)}\bigg{]}\\
% \end{align*}
% \begin{align*}
& =-\sum_{i,j}d_{\mu}(i)  \big{(}\g\phi(i)\phi(j)^T  p_{\pi}(j|i)-(1+\eta)\phi(i)\phi(i)^T\big{)}\\
& =\Phi^TD_{\mu}((1+\eta)I-\g P_{\pi})\Phi.
\end{align*}
Similarly 
\begin{align*}
    b & = \E_{0}[b_n] = \E_{0}[\rho_n r_n\phi(s_n)] = \Phi^TD_{\mu}r_{\pi}.
\end{align*}
\end{proof}

\begin{definition}
A  $d\times d$ matrix $M$ is positive definite if for all $0 \neq y \in \R^{d}$, $y^TMy > 0.$
\end{definition}
\begin{lemma}\label{pd_lemma}
Given a $d\times d$ matrix $M$, $M$ is positive definite iff the symmetric matrix $S=M+M^T$ is positive definite.
\end{lemma}
\begin{proof}
For $0 \neq y \in \R^{d}$ observe that
\begin{align*}
    y^TSy=y^TMy+y^TM^Ty=2y^TMy,
\end{align*}
since $(y^TMy)^T=y^TMy$ as both are scalars and $y^TM^Ty=(y^TMy)^T$.
Hence $S$ is positive definite if and only if $M$ is positive definite
\end{proof}
\begin{theorem}
\label{eta_bound}
Suppose $M=D\left((1+\eta)I-\gamma P\right)$ where $D$ is a diagonal matrix with positive diagonal entries, $P$ is a Markov matrix and $\eta\geq\max\left(\displaystyle \max_{i} \frac{\gamma d^Tp(.|i)}{d_i}-1,0\right)$ and $0<\gamma<1$ are positive constants. Then $M=[m_{ij}]$ is positive definite.
\end{theorem}
\begin{proof}
Consider the symmetric matrix $S=M+M^T$. From Lemma \ref{pd_lemma}, it is enough to show that $S$ is positive definite. Since $S$ is symmetric it is diagonalizable. Therefore it is enough to show that the eigen-values of $S$ are positive. From the Gershgorin circle theorem (see \cite{golub4matrix}) for any eigen value $\lambda$ of $S$ there exists $i$ such that
\begin{align*}
    &|\lambda -2m_{ii}|\leq \sum_{j\neq i}|m_{ij}|+\sum_{j\neq i}|m_{ji}|\\
    \implies & \lambda \geq 2m_{ii}- \sum_{j\neq i} |m_{ij}|-\sum_{j\neq i}|m_{ji}|.
\end{align*}
Now $m_{ii}=d_{i}\left((1+\eta)-\g p(i|i)\right)$ and for $i \neq j$ we have $ m_{ij}=-d_{i}\g p(j|i)$. Therefore 
$m_{ii}-\sum_{j\neq i}|m_{ij}|=(1+\eta-\gamma)d_{i}$ and $m_{ii}-\sum_{j\neq i}|m_{ji}|=\left((1+\eta)d_{i}-\gamma d^Tp(i|.)\right),$
\begin{align*}
\implies &\lambda \geq (1+\eta-\gamma)d_{i}+\left((1+\eta)d_{i}-\gamma d^Tp(i|.)\right)>0
\end{align*}
from the hypothesis $\eta>\displaystyle \max_{i} \frac{\gamma d^Tp(i|.)}{d_i}-1$. We see that every eigen-value of $S>0$ i.e., $S$ is positive definite. Hence $M$ is positive definite.
In particular, given the behaviour policy $\mu$ and the target policy $\pi$, there exists $\eta >0$ such that $A=\Phi^TD_{\mu}\left((1+\eta)I-\gamma P_{\pi}\right)\Phi$ is positive definite.
\end{proof}
%%%%%%%%%%%%%%%%%%%%%%%%%%%%%%
To describe the point of convergence of our algorithm consider for a given policy $\mu$ and a parameter $\eta$, $T^{\eta}_{\mu}: \R^{|S|}\rightarrow \R^{|S|}$ as $T^{\eta}_{\mu}=\frac{1}{1+\eta}T_{\mu}$. We state and prove the following properties about $T^{\eta}_{\mu}.$
\begin{lemma}
\label{T_approx}
$T^{\eta}_{\mu}$ is a $\|.\|_{\infty}$-contraction and converges point-wise to $T_{\mu}$ as $\eta \rightarrow 0.$
\end{lemma}
\begin{proof}
From the definition $T^{\eta}_{\mu}=\frac{1}{1+\eta}T_{\mu}.$ For any $V \in \R^{|S|},$
\begin{align*}
    T^{\eta}_{\mu}V=\frac{1}{1+\eta}T_{\mu}V \rightarrow T_{\mu}V \text{ as } \eta \rightarrow 0.
\end{align*}
It is easy to see that for any $V, W \in \R^{|S|}$,
\begin{align*}
 \|T^{\eta}_{\mu}V-T^{\eta}_{\mu}W\|_{\infty}
=&\frac{\g}{1+\eta}\|P_{\mu}(V-W)\|_{\infty}\\
\leq &\frac{\g}{1+\eta}\|V-W\|_{\infty}.
\end{align*}
Hence $T^{\eta}_{\mu}$ is $\|.\|_{\infty}$- contraction.
\end{proof}
% We note the following observation that describes the point of convergence, for any $V \in \R^{|S|}, T^{\eta}_{\mu}V=\frac{r_{\mu}}{1+\eta}+\frac{\g}{1+\eta}P_{\mu}V$ is a Bellman operator for a given policy with a scaled single stage reward and discount factor.
\section{About the Point of Convergence}
The algorithm converges to the point $\theta^*$ such that $b-A\theta^*=0.$ Now
\begin{align*}
    & b-A\theta^*=0 \\
\implies &  \Phi^TD_{\mu}\left((1+\eta)I-\gamma P_{\pi}\right)\Phi\theta^*= \Phi^TD_{\mu}r_{\pi}\\   
\implies &  \Phi^TD_{\mu}\left(I-\frac{\gamma}{1+\eta} P_{\pi}\right)\Phi\theta^*= \Phi^TD_{\mu}\frac{r_{\pi}}{1+\eta}\\ \implies &  \Phi^TD_{\mu}\Phi\theta^*= \Phi^TD_{\mu}\left(\frac{r_{\pi}}{1+\eta}+\frac{\gamma}{1+\eta} P_{\pi}\Phi\theta^*\right)\\
\implies & \Phi\theta^*= \Phi(\Phi^TD_{\mu}\Phi)^{-1}\Phi^TD_{\mu}\left(\frac{r_{\pi}}{1+\eta}+\frac{\gamma}{1+\eta} P_{\pi}\Phi\theta^*\right)\\
\implies & \Phi\theta^*=\Pi_{D_{\mu}} T_{\pi}^{\eta}\Phi\theta^*,
\end{align*}
where $\Pi_{D_{\mu}}=\Phi(\Phi^TD_{\mu}\Phi)^{-1}\Phi^TD_{\mu}$ is the projection operator that projects any $V \in \R^{|S|}$ to the subspace $\{\Phi r|r \in \R^d\}$ with respect to the norm $\|.\|_{D_{\mu}}.$
Hence we observe that, similar to online on-policy TD, 
our online off-policy TD is a projected stochastic fixed point iteration with respect to the perturbed Bellman operator $T^{\eta}_{\pi}.$ 
% However the projection $\Pi_{D_{\mu}}=\Phi(\Phi^TD_{\mu}\Phi)^{-1}\Phi^TD_{\mu}$ on the subspace $\{\Phi r|r \in \R^d\}$ is with respect to the norm $\|.\|_{D_{\mu}}.$

\begin{remark}
Note that the bound derived for $\eta$ in Theorem \ref{eta_bound} is a sufficient but not a necessary condition. If the value of $\eta$ is large, the algorithm converges but to a poorly approximated solution. Therefore, in experiments, we select the value of $\eta$ that is large enough to ensure convergence and small enough to ensure that approximation is reasonable. 
\end{remark}

% \begin{remark}
% Note that the bound derived for $\eta$ in Theorem \ref{eta_bound} is a sufficient but not a necessary condition. If the value of $\eta$ is large, the algorithm converges but to a distant solution from the optimal solution. Therefore, in experiments, we select the value of $\eta$ that is large enough to ensure convergence and small enough to ensure that the solution is close to the optimal solution -----. 
% \end{remark}

%We also note that large values of $\eta$ ensure convergence. On the other hand smaller values of $\eta$ ensure that $T_{\pi}^{\eta} $ is close to $T_{\pi}$
% 
% \begin{lemma}
% $T^{\eta}_{\mu}$ preserves partial order in the policy space.
% \end{lemma}
% \begin{proof}
% Given two policies $\mu_{1}$ and $\mu_{2}$ with $\mu_{1} \succeq \mu_{2}$, it is enough to show that $V^{\eta}_{\mu_{1}} \geq V^{\eta}_{\mu_{2}}$ where $V^{\eta}_{\mu_{1}}, V^{\eta}_{\mu_{2}}$ are the fixed points of $T^{\eta}_{\mu_{1}}$ and $T^{\eta}_{\mu_{2}}$ respectively. Now by definition of the partial order between policies we have
% \begin{align*}
%     \mu_{1} \succeq \mu_{2} \iff V_{\mu_{1}} \geq V_{\mu_{2}}.
% \end{align*}
% Now for any $n$
% \begin{align*}
%     &V_{\mu_{1}} \geq V_{\mu_{2}}
% \implies T^{n}_{\mu_{1}} V_{\mu_{1}} \geq T^{n}_{\mu_{2}} V_{\mu_{2}}\\
% \implies & \frac{T^{n}_{\mu_{1}} V_{\mu_{1}}}{(1+\eta)^{n}} \geq \frac{T^{n}_{\mu_{2}} V_{\mu_{2}}}{(1+\eta)^{n}}\\
% \implies & V^{\eta}_{\mu_{1}} \geq V^{\eta}_{\mu_{2}}
% \end{align*}
% \end{proof}
%

%%%%%%%%%%%%%%%%%%%%%%%%%%%%%%%

\section{Experiments and Results}
\begin{figure}
    \centering
    \includegraphics[scale = 0.4]{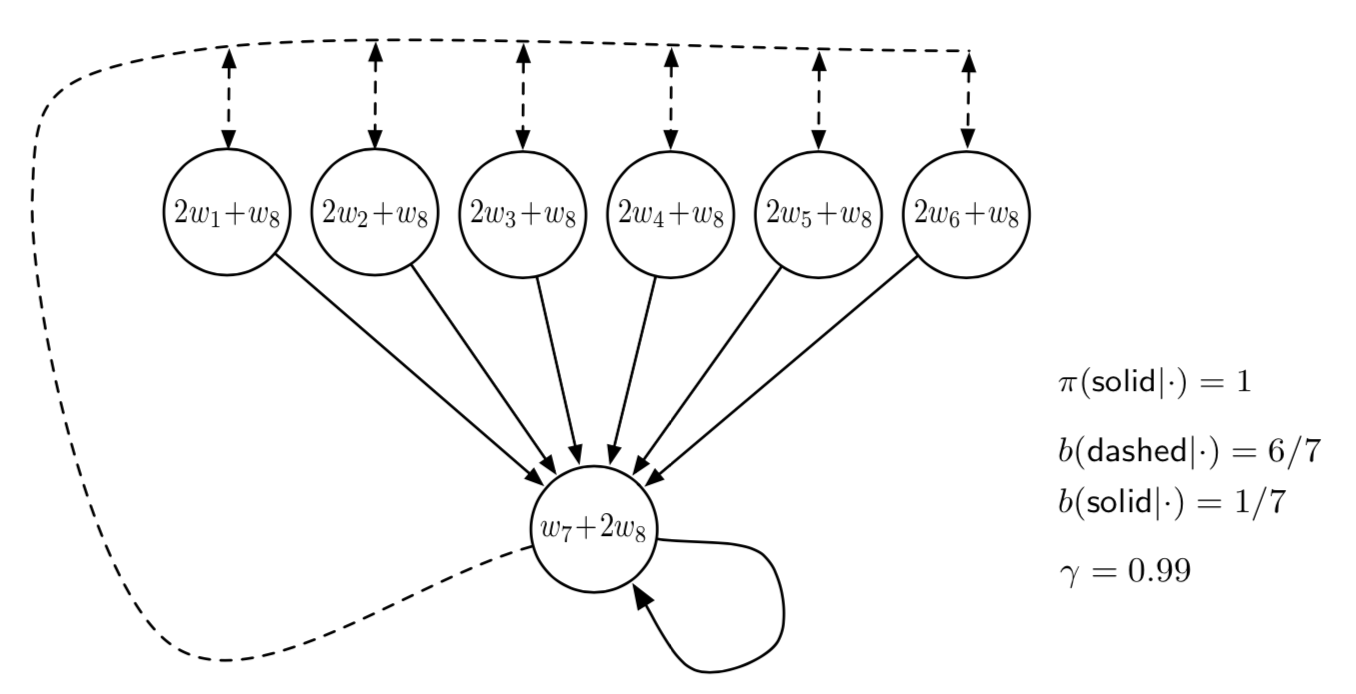}
    \caption{Baird's Counterexample. Figure taken from \cite{baird}}
    \label{bairdsEx}
\end{figure}
In this section, we describe the performance of our proposed algorithm on three tasks. We first perform experiments on two benchmark counter-examples for off-policy divergence. Finally, we perform our experiment on a 3-state MDP example and analyze various properties of our proposed algorithm \footnote{The implementation codes for our experiments is available at: \url{https://github.com/raghudiddigi/Off-Policy-Convergent-Algorithm} }. The evaluation metric considered is Root Mean Square Error (RMSE) defined as:
\begin{align}
    RMSE(\theta) = \sqrt{d_\mu(s)(V_\pi(s) - \widehat{V}_\theta(s))^2},
\end{align}
where $\theta$ is the parameter that is used to approximate the value function, $d_\mu$ is the stationary distribution associated with the behavior policy $\mu$, $V_\pi$ is the exact value function of the target policy $\pi$ and $\widehat{V}_\theta$ is the approximate value function that is estimated. We perform $10$ independent runs and present the average of RMSE obtained on all the three experiments. For comparison purposes, we also implement Emphatic TD (ETD(0)) algorithm \cite{sutton2016emphatic} and a gradient-family algorithm, linear TD with
gradient correction (TDC) \cite{sutton2009fast}.

First, we consider the ``$\theta \xrightarrow{} 2\theta$" example (\cite{tsitsiklis1996feature}, Section 3 of \cite{sutton2016emphatic}). In this example, there are two states - $1$ and $2$ and two actions - 'left' and 'right'. Left action in state $1$ results in state $1$, while right action results in state $2$. Similarly, right action in state $2$ results in state $2$ and left action results in state $1$. The target policy is to take right in both the states, whereas behavior policy is to take left and right actions with equal probability in both the states. The value function is linearly approximated with one feature. The feature of state $1$ is $1$ and that of state $2$ is $2$. The discount factor is taken to be $0.9$. The update parameter $\theta$ is initialized to $1$ and the $\eta$ for our algorithm is taken to be $1$. The step-size for the algorithms is held constant at $0.01$. In Figure \ref{fig1}, we show the performance of algorithms over $10000$ iterations. We can see that the standard off-policy TD(0) diverges whereas the other three algorithms including our proposed perturbed off-policy TD(0) converges to a point where the RMSE is zero. 
\begin{figure}
    \centering
    \includegraphics[scale = 0.5]{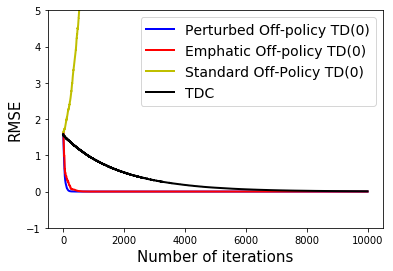}
    \caption{Performance of algorithms on ``$\theta \xrightarrow{} 2\theta$". RMSE is the value averaged across $10$ independent runs}
    \label{fig1}
\end{figure}

Next, we consider the ``7-star" example, first proposed in \cite{baird1995residual}. This is completely described in Figure \ref{bairdsEx} \cite{baird}. There are $7$ states represented as circles. The expression inside the circle $i$ represents the linear approximation of the state $i$. The policy $\pi$ in Figure \ref{bairdsEx} represents the target policy and $b$ represents the behavior policy. We run all the algorithms, i.e., standard off-policy TD(0), Emphatic off-policy TD(0), TDC and our algorithm, Perturbed off-policy TD(0) for $1000000$ iterations. The step-size for the algorithms is set to $0.0001$ \footnote{We have run experiments with three other step-sizes and included it in our supplementary material. Please find them at: \url{https://github.com/raghudiddigi/Off-Policy-Convergent-Algorithm/blob/master/Supplementary.pdf}}. From Figure \ref{fig2}, we can see that our perturbed off-policy TD converges to the exact solution while the Emphatic TD(0) appears to oscillate. On the other hand, the TDC algorithm appears to converge slowly. Moreover, it is known that standard off-policy diverges for this example, which can also be observed from Figure \ref{fig2}. 
%Moreover, in Figure \ref{}, we show the performance of our algorithm over all the $10$ runs. 
\begin{figure}
    \centering
        \includegraphics[scale = 0.5]{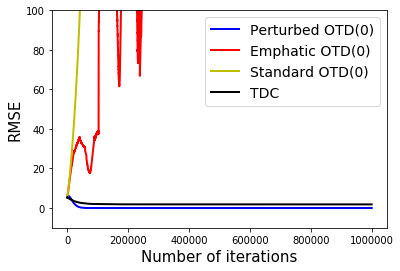}
    \caption{Performance of algorithms on ``Baird's Counter-example". RMSE is the value averaged across $10$ runs}
    \label{fig2}
\end{figure}

Finally, we construct an MDP as follows. There are $3$ states and $2$ actions - 'left' and 'right' possible in each state. The 'Left' action in states $1$ and $2$ leads to state $1$. And the 'right' action in states $2$ and $3$ leads to state $3$. Finally 'left' action in state $3$ leads to state $2$. The single-stage rewards in all transitions is taken to be $1$ and the discount factor is $0.9$. The target policy $\pi =[ [0,1],[0.5,0.5],[1,0]]$ and the behavior policy $\mu = [[0.9,0.1],[0.5,0.5],[0.1,0.9]]$ (where the first component represents the probability to take 'left' and the second component represents the probability to take 'right'). The feature vectors of the three states are $[1,0],[1,1],[0,1]$ respectively. The step-size for the algorithms is set to $0.0001$. We run all the algorithms for $1000000$ iterations. From Figure \ref{fig3}, we can see that perturbed off-policy TD(0) converges. For this experiment, the best possible RMSE is $2.548$ and our proposed algorithm achieves $2.97$. 
\begin{figure}
    \centering
    \includegraphics[scale = 0.5]{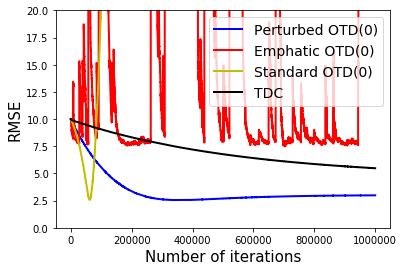}
    \caption{Performance of algorithms on a 3-state MDP. RMSE is the value averaged across $10$ runs.The best possible RMSE for this MDP is $2.548$ }
    \label{fig3}
\end{figure}

In the experimental setting above, the value of $\eta$ is set to $0.5$. In Figure \ref{fig4}, we run our algorithm with two other values of $\eta = 0.4$ and $0.6$ respectively. We observe that, for $\eta = 0.4$, convergence is not guaranteed as this $\eta$ correction is not enough. On the other hand, for $\eta = 0.6$, convergence is ensured. However, the converged solution is not close due to the over-correction. Hence, it is to be noted that an optimal value of $\eta$ is desired for ensuring the convergence and near-optimal solution at the same time (recall that a higher value of $\eta$ is enough to ensure the convergence alone). 
\begin{figure}
    \centering
    \includegraphics[scale = 0.5]{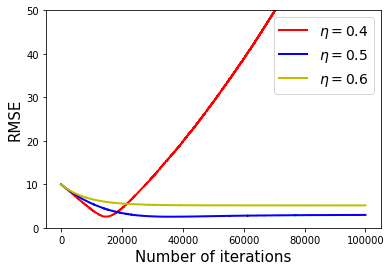}
    \caption{Performance of our proposed algorithm with three different $\eta$ values }
    \label{fig4}
\end{figure}

\begin{remark}
It has to be noted that the objective of the experiments is to show that our proposed algorithm mitigates the divergence problem of the off-policy TD algorithm. Moreover, if we choose a good value of $\eta$, it ensures that the algorithm converges to a solution closer to the optimal solution. At this point, we do not make any claims about the quality of the converged solution compared to the Emphatic TD(0) and TDC algorithms. We have seen that our proposed algorithm performed better than Emphatic TD and TDC in the last two examples. Further empirical analysis is needed to compare the quality of the converged solution with Emphatic TD(0), TDC as well as other off-policy algorithms in the literature. 
\end{remark}

\section{Conclusions and Future Work}
In this work, we have proposed an off-policy TD algorithm for mitigating the divergence problem of the standard off-policy TD algorithm. Our proposed algorithm makes use of a penalty parameter to ensure the stability of the iterates. We have then proved that this addition of penalty parameter makes the matrix $A$ positive definite, which in turn ensures the convergence. Finally, we empirically show the convergence on benchmark counter-examples for off-policy divergence. 

As seen from the experiments, the choice of $\eta$ is critical for our algorithm. The lower-bound that we have provided in our analysis is not tight and coming up with a tight bound is an interesting future direction. Also, in future, we would like to extend our algorithm to include eligibility traces and study its applications on real world problems.

%%%%%%%%%%%%%%%%%%%%%%%%%%%%%%%%%%%%%%%%%%%%%%%%%%%%%%%%%%%%%%%%%%%%%%%%%%%%%%%%

\bibliographystyle{IEEEtran}
\bibliography{references}
\end{document}